%% file: main.tex
\documentclass{article}
\input{header.tex}

\title{\bfseries Kernel Embeddings \\and the Separation of Measure Phenomenon}

\author[1]{\textsc{Leonardo V.\ Santoro}}
\author[2]{\textsc{Kartik G.\ Waghmare}}
\author[1]{\textsc{Victor M.\ Panaretos}}

\affil[ ]{\footnotesize{
\texttt{leonardo.santoro@epfl.ch}  
\qquad\,\,
\texttt{kartik.waghmare@stat.math.ethz.ch}
\qquad\,\,
\texttt{victor.panaretos@epfl.ch}}
\normalsize}

\affil[1]{Institut de Math\'ematiques\\École Polytechnique Fédérale de Lausanne}
\affil[2]{Departement Mathematik, ETH Zürich}
  
\begin{document}

\maketitle

\begin{abstract}
We prove that kernel covariance embeddings lead to information-theoretically perfect separation of distinct continuous probability distributions. 
In statistical terms, we establish that testing for the \emph{equality} of two non-atomic (Borel) probability measures on a locally compact uncountable Polish space is \emph{equivalent} to testing for the \emph{singularity} between two centered Gaussian measures on a reproducing kernel Hilbert space.  The corresponding Gaussians are defined via the notion of kernel covariance embedding of a probability measure, and the Hilbert space is that generated by the embedding kernel.  Distinguishing singular Gaussians is structurally simpler from an information-theoretic perspective than non-parametric two-sample testing, particularly in complex or high-dimensional domains. This is because singular Gaussians are supported on essentially separate and affine subspaces.  Our proof leverages the classical Feldman-H\'{a}jek dichotomy, and shows that even a small perturbation of a continuous distribution will be maximally magnified through its Gaussian embedding. This ``separation of measure phenomenon'' appears to be a blessing of infinite dimensionality, by means of embedding, with the potential to inform the design of efficient inference tools in considerable generality. The elicitation of this phenomenon also appears to crystallize, in a precise and simple mathematical statement, a core mechanism underpinning the empirical effectiveness of kernel methods.
\end{abstract}


\section{Introduction}
Two-sample hypothesis testing is a foundational statistical problem, arguably as old as the discipline itself. It enables the researcher to determine whether two populations differ significantly with respect to certain quantitative features, from representative data. Its origins can be traced at least as far back as {Karl Pearson}'s \cite{pearson1900} chi-squared test, initially developed for analyzing biological and clinical data, and William Sealy Gosset's  t-test \cite{student1908} developed to compare the yields of different crop treatments. These two tests are still widely taught and used today, and they established two-sample testing as a practical tool in experimental research. Two-sample hypothesis testing also forms the basis of the Analysis of Variance (ANOVA), which generalizes the concept to multi-sample comparisons. By the mid-20th century, the focus expanded beyond tests based on parametric models, to encompass situations wherethe researcher cannot or would rather not commit to a stringent model specification. Tests based on ranks, either in paired settings \citep[Wilcoxon]{wilcoxon1945} or unpaired settings \cite[Mann \& Whitney]{mannwhitney1947} exploited the order of the real line and offered robustness, becoming essential in clinical research, genetics, and industrial quality control. Methods such as the {Kolmogorov-Smirnov test} \cite{kolmogorov1933,smirnov1948table} also avoided parametric assumptions by means of what today is called an invariance principle. Such procedures fall under the umbrella of what has come to be known as \emph{non-parametric testing}.

Non-parametric testing becomes particularly challenging when the probability distributions involved are defined on a high-dimensional and/or complex domain since a much larger number of features need to be compared. The primary issue is that the non-parametric alternative hypothesis is too vague: there are simply far too many ways in which two distributions can differ in high dimensions and/or complex domains. 

Without knowing which kinds of deviations to target, it becomes difficult to
optimize the choice of test statistic. 
Yet such data sets are increasingly becoming the norm, not only in statistical applications but especially in the context of machine learning \cite{recht2019imagenet,hendrycks2019benchmarking,torralba2011unbiased}. Consequently, nonparametric tests often either aim to probe for the intrinsic structure of the data set (e.g., graph-based methods \cite{Henze88, chen2017new} and depth-based techniques \cite{zuo2000general, Wilcox01042003, Shojaeddin12}) or to embed the data in a new space, where differences are hopefully amplified (kernel-based methods \cite{ aronszajn1950theory, berlinet2011reproducing,gretton2012kernel}). Kernel methods, in particular, Maximum Mean Discrepancy (MMD) and its variants \cite{gretton2012optimal, liu2020learning,hagrass2024spectral, eric2007testing,shekhar2022permutation,chatterjee2024boosting}, are generally seen to provide state-of-the-art performance, and are used extensively in complex learning contexts. It is widely understood that this stems from the effectiveness of the so-called “kernel trick”, whereby applying linear methods to nonlinear embeddings of the data into an infinite-dimensional space, rather than the original data, leads to better performance. Yet a precise mathematical statement that can transparently explain why has apparently remained elusive.

The contribution of this paper is twofold. First, we elicit a ``separation of measure phenomenon'' in the form of a clean and rigorous  mathematical statement, crystallizing why kernel methods should perform so well (Theorem~\ref{thm:main} and Corollary~\ref{cor:main}). Second, we demonstrate that current implementations (based on mean embedding alone) do not make use of the separation of measure phenomenon and can thus miss the full potential of the kernel trick (Proposition~\ref{prop:main:mean}). And, that a more refined use of embeddings holds considerable further potential for two-sample testing (Theorem~\ref{thm:operat}), and possibly in other inferential settings. The key insight is that, when used appropriately, the kernel trick transforms (perhaps subtle) differences between arbitrary distributions into maximally separated Gaussian measures on the embedding space -- the Gaussians with moments corresponding to the embedding moments. Notably, our results hold in considerable generality, requiring only that the domain of the distributions in question be a locally compact uncountable Polish space, and the distributions be non-atomic (indeed, non-atomicity of the distributions automatically entails uncountability of the domain). Consequently, the result could serve as a basis for the design of powerful inference tools in a wide range of contexts. The obvious field of application is testing, but one can appreciate the potential in classification, generalisation bounds, and variational inference, to mention but a few. As a proof of concept, we mention that in follow-up work, Santoro \& Panaretos \cite{santoroLRT25} provide an in-depth study of tests exploiting the separation of measure phenomenon identified in the present paper, and report remarkable empirical gains in power relative to state-of-the-art testing methods.

\section{High-Level Overview of Our Contributions}

The key tool underpinning kernel methods, such as MMD, is the concept of \textit{kernel mean embedding}, which provides a non-parametric representation of probability distributions by mapping them into an RKHS. Formally, given a locally compact Polish  space $\cX$, let $k : \cX\times \cX\to \bbR$ be a bounded positive semidefinite kernel (henceforth \emph{kernel}), with the RKHS $\cH$, and consider the corresponding feature map $\bx \mapsto k_{\bx} \in \cH$, where $k_\bx = k(\bx,\cdot)$. Given a probability distribution $ \bbP $, the corresponding \emph{mean embedding} maps $ \bbP $ to a \textit{point} ({vector}) in the RKHS, namely the embedding's first moment, as follows:
\begin{equation}
    \label{eq:meanemb}
    \bbP\mapsto \m_{\bbP} := \int k_\bx \, d\bbP(\bx).
\end{equation}
This embedding provides a representation of the distribution in an infinite-dimensional space and enables comparisons between distributions in terms of the (flat and linear) Hilbert space geometry. In fact, it facilitates various statistical tasks, and has been extensively used for machine learning and data science purposes: see \cite{muandet2017kernel} for a review. In the same vein,\textit{ kernel covariance embeddings} \cite{fukumizu2004dimensionality, gretton2007kernel, bach2022information}  extend this concept to capture the second moment of the embedding. Just as kernel mean embeddings map distributions to elements of the RKHS, kernel covariance embeddings map to  linear operators on the RKHS: the (uncentered) 
 covariance embedding $\S_{\bbP}$ of a probability measure $\bbP$ is defined as
\begin{equation}\label{eq:covemb}
    \bbP\mapsto \S_{\bbP} := \int k_\bx\otimes k_\bx \, d\bbP(\bx) 
\end{equation} 
and constitutes a self-adjoint, positive-semidefinite linear operator of finite trace from $\cH$ to $\cH$. 

Once we have a mean embedding $\m_\bbP$ and a covariance embedding $\S_\bbP$, a natural next step is to associate to \( \bbP \) a \emph{Gaussian measure on the RKHS $\cH$} with those two moments, as in
$$ \bbP\mapsto \mathcal{N}(\m_{\bbP}, \S_{\bbP}).$$
We refer to this as the \textit{kernel Gaussian embedding} of the distribution $\bbP$, or simply as the \emph{Gaussian embedding}. Note that one can also define the centered embedding, $ \bbP\mapsto \mathcal{N}(\mathrm{0}, \S_{\bbP})$, and this distinction will play an important role. While conceptually straightforward, these embeddings will allow us to go beyond the functional structure of the RKHS, and to exploit the properties of Gaussian measures on Hilbert spaces: it is not so much the linear geometry of the embedding space, but rather the information geometry of these Gaussian embeddings that is key.

In particular, we will demonstrate that kernel embeddings make it possible to reformulate the classical two-sample problem -- testing for the equality of distributions -- in terms of testing the \emph{mutual singularity} of the corresponding Gaussian embeddings. Recall that two probability measures $\mu,\nu$ on a measurable space $\mathcal{X}$ are \emph{mutually singular} (denoted $\mu\perp\nu$) if they ``separate": if there exists a measurable set $A$ such that $\mu(A)=\nu(A^c)=0$, so that $A$ carries 
all the mass of $\nu$ but none of the mass of $\mu$. It follows that two mutually singular probability measures have supports that are \emph{essentially disjoint}, in that their intersection is a null set under at least one of the two probability measures.  Of course, two probability measures can be distinct, without being mutually singular. Mutual singularity represents an extreme case where the two measures are maximally separated in an information-theoretic sense (neither measure has density with respect to the other). On the other end, we say that $\mu$ and $\nu$ are \emph{equivalent} (denoted $\mu\sim\nu$) when they share the same support: for any measurable $B$, we have $\mu(B)=0\iff \nu(B)=0$. A fundamental result in the theory of Gaussian measures states that Gaussians are either mutually equivalent or mutually singular, with no intermediate case, and the dichotomy is governed by precise criteria.

Leveraging this result, known as the Feldman-H\'{a}jek theorem, our main contribution is to show that kernel Gaussian embeddings lead to separation of measure in the following sense:

\begin{center} \emph{two non-atomic probability measures are distinct\\[0.5em] if and only if\\[0.5em] the corresponding kernel Gaussian embeddings are \emph{mutually singular}}
\end{center}

The result holds true whether we use the centered (Theorem~\ref{thm:main}) or the uncentered embedding (Corollary~\ref{cor:main}):
\begin{equation}\label{eq:main}
       \begin{split}
    \bbP \neq \bbQ \quad &\Longleftrightarrow \quad \mathcal{N}(\m_{\bbP}, \S_{\bbP}) \perp \mathcal{N}(\m_{\bbQ}, \S_{\bbQ})
    \\ \quad &\Longleftrightarrow \quad \mathcal{N}(\mathrm{0}, \S_{\bbP}) \perp \mathcal{N}(\mathrm{0}, \S_{\bbQ})
\end{split}
\end{equation}
Furthermore, we show that it is the covariance component of the embedding that elicits the separation-of-measure phenomenon, whereas the mean embedding alone does not suffice (Proposition \ref{prop:main:mean}).  Intuitively, kernel Gaussian embedding ``sharpens'' the alternative hypothesis by separating the embedded measures: it transforms it from a question of whether two distributions deviate, which can be rather nuanced in a nonparametric setting, to the considerably more transparent question of whether two Gaussians have essentially separate supports. Beyond non-atomicity, the original measures can be arbitrary Borel probability measures on a general locally compact uncountable Polish space, so our result holds very generally. This illustrates, in a precise sense, a kind of ``blessing of infinite dimensionality": suitable kernel embedding into an infinite-dimensional RKHS  separates continuous distributions perfectly, even when their differences are arbitrarily subtle. 

Of course, the obtained maximal separation comes at the cost of the embedded measures being supported on (subspaces of) an infinite-dimensional Hilbert space. Nevertheless, these measures are Gaussian, so it suffices to look at their empirical means/covariances -- which are $\sqrt{n}$-estimable in \emph{dimension independent} fashion via their empirical counterparts. And, once ``we know what to look for" we can target the alternative at the level of Gaussian embedding via the right information-theoretical tools: the corresponding Gaussian relative entropy (equivalently, a \emph{Gaussian likelihood ratio}), which converges or diverges according to whether we are under the null or alternative regime  (Theorem~\ref{thm:operat}). This suggests  that kernel-based methods hold the potential to yield even greater statistical efficiency when informed by our results.

\section{Background and Preliminaries}

We begin by summarizing basic notions in functional analysis and measure theory that are key to developing our results.

 \subsection*{Operator Theory}
Let $( \cH, \langle\cdot,\cdot\rangle_{\cH} )$ be a separable Hilbert space with 
induced norm $\|\cdot \|_{\cH}\::\: \cH \rightarrow [0,\infty)$, with $\mathrm{dim}(\cH)\in \bbN\cup\{\infty\}.$
Given $f,g\in\cH$, their \textit{tensor product} $f\otimes g \::\: \cH\rightarrow\cH$ is the linear operator defined by:
$$
(f\otimes g)u = \langle g,u\rangle_{\cH} f,\qquad u\in\cH.
$$ 
Given Hilbert spaces $\cH_1,\cH_2$ and a linear operator $\bA:\cH_1\rightarrow\cH_2$, we define its adjoint as the unique operator $\bA^*:\cH_2\rightarrow\cH_1$ such that $\langle \bA u, v \rangle_{\cH_2} = \langle u, \bA^* v \rangle_{\cH_1}$ for all $u\in \cH_1, v\in \cH_2$. We say that an operator $\bA:\cH\to\cH$ is \textit{self-adjoint} if $\bA=\bA^*$. 
We say that $\bA$ is \textit{non-negative definite} (or \textit{positive semidefinite}), and write $\bA\succeq 0$ if it is self-adjoint, and satisfies $\langle\bA h, h\rangle_{\cH} \geq 0$ for all $h\in\cH$. When the inequality is strict for all $x\in\cH\setminus\{0\}$ we call  $\bA$ \textit{positive definite} and write $\bA\succ 0$. 
We say that $\bA$ is \emph{compact} if for any bounded sequence $\{h_n\}\subset  \cH$, $\{\bA h_n\}\subset  \cH$ contains a convergent subsequence.
If $\bA$ is a non-negative, compact operator, then there exists a unique non-negative operator denoted by $\bA^{\sfrac{1}{2}}$ that satisfies $( \bA^{\sfrac{1}{2}} )^2 = \bA$.
The \textit{kernel} of $\bA$ is denoted by $\ker(\bA) = \{h\in\cH\::\: \bA h= 0\}$, and its \textit{range} by $\range(\bA)=\{\bA h \::\: h \in \cH\}$. 
We denote the \textit{trace} of an operator $\bA$, when defined, by $\trace(\bA)= \sum_{i\geq 1}\langle \bA e_i,e_i\rangle_{\cH} $, where $\{e_i\}_{i\geq 1}$ is an (arbitrary) Complete Orthonormal System (CONS) of $\cH$. We write
\begin{align*}
 & 
 \|\bA\|_{\text{op}(\cH)}:=\sup_{\|h\|_{\cH}=1}\|\bA h\|_{\cH}
  \\ &
  \|\bA\|_{\textnormal{HS}(\cH)}:= \sqrt{\trace(\bA^*\bA)}, 
 \\ &
 \|\bA\|_{\mathrm{tr}(\cH)}:= \trace(\sqrt{\bA^*\bA})
\end{align*}
 for the \emph{operator norm}, \emph{Hilbert-Schmidt norm}, and \emph{trace norm},
 respectively. An operator $\bA$ is said to be \emph{Hilbert-Schmidt} if $\|\bA\|_{\textnormal{HS}(\cH)} < \infty$ and trace-class if $\|\bA\|_{\mathrm{tr}(\cH)} < \infty$. 
One always has
 $
  \|\bA\|_{\text{op}(\cH)} \leq
 \|\bA\|_{\textnormal{HS}(\cH)} \leq \|\bA\|_{\mathrm{tr}(\cH)}
 $. We write $\bI$ for the identity operator on $\cH$.

\subsection*{Reproducing Kernel Hilbert Spaces} Let $\mathcal{X}$ be a locally compact Polish space. Consider a positive semidefinite kernel $k: \mathcal{X} \times \mathcal{X} \to \bbR$. Heuristically, the Reproducing Kernel Hilbert Space (RKHS) associated with  $k$, denoted $\cH = \cH(k)$,
is the Hilbert space of
$f: \mathcal{X}\to \bbR$ spanned by (possibly infinite) linear combinations of \emph{feature vectors} $\{k(\cdot,x_i)\}$.  Formally, let $\cH^0$ be the set of all finite linear combinations of feature vectors:
$$
\cH^0:=\operatorname{span}\{k(\cdot, x): x \in \mathcal{X}\}
$$
One can turn $\cH^0$ into a pre-Hilbert space by defining an inner product as follows: Given $f,g \in \cH^0$, with
\begin{equation*}
\begin{split}
f:=\sum_{i=1}^n a_i k\left(\cdot, x_i\right)\quad \text{ and } \quad g:=\sum_{j=1}^m b_j k\left(\cdot, y_j\right)\\ 
\text{then: } \qquad \langle f, g\rangle_{\cH^0}:=\sum_{i=1}^n \sum_{j=1}^m a_i b_j k\left(x_i, y_j\right),
\end{split}
\end{equation*}
where $a_1, \ldots, a_n$, $b_1, \ldots, b_m \in \mathbb{R}$ and $x_1, \ldots, x_n$, $y_1, \ldots, y_m \in \mathcal{X}$ for some  $n, m \in \bbN$.
The RKHS associated with $k(\cdot,\cdot)$ is then defined as the completion of $\cH^0$ with respect to $\|\cdot\|_{\cH^0}$, i.e, $\cH:=\overline{\cH^0}$.

 The distinguishing feature of an RKHS is that evaluation functionals $f \mapsto f(x)$ are continuous and satisfy the \emph{reproducing property}:
   $$
   f(x) = \langle f, k_x \rangle \quad \text{for all} \quad x \in \mathcal{X} \quad \text{and} \quad f\in \cH.
   $$
   where $k_x = k(x,\cdot)$.
   In other words, the value of the function $f$ at any point $x \in \mathcal{X}$ depends continuously on $f$ in the RKHS norm and can be recovered by taking the inner product of $f$ with the kernel function $k_x$.

An important property that a kernel may possess is $C_{0}(\mathcal X)$-\emph{universality}— or, for simplicity,  simply  \emph{universality}—which means that the associated RKHS is dense in $C_{0}(\mathcal X)$ with respect to the uniform norm.

\paragraph*{Mean and Covariance Embeddings.}
Kernel embeddings have emerged as powerful tools in machine learning and statistical inference. The key idea is to map probability measures to vectors or functions in a reproducing kernel Hilbert space (RKHS), thereby enabling the application of linear or multivariate methods directly to distributions \cite{muandet2017kernel}.

Let $\mathcal{X}$ be a locally compact Polish space, and denote by $\cP(\mathcal{X})$ the set of Borel probability measures on $\mathcal{X}$. Given a bounded positive semidefinite kernel $k : \mathcal{X} \times \mathcal{X} \to \mathbb{R}$ with associated RKHS $\cH$, the \emph{kernel mean embedding} of a measure $\bbP \in \cP(\mathcal{X})$ is defined (as in \eqref{eq:meanemb}) by the unique element $\m_\bbP \in \cH$ satisfying

$$
\langle \m_\bbP, f\rangle_{\cH} = \int_{\cX} f(u)~d\bbP(u), \qquad \forall f \in \cH.
$$

If further to being bounded $k$ is also \emph{universal} on $\cX$, then the mapping $\bbP \mapsto \m_\bbP$ is injective, so the embedding fully characterizes the distribution (the kernel is \emph{characteristic}).

In a similar spirit, the \emph{(uncentered) covariance kernel embedding} (or kernel covariance operator) of a measure $\bbP \in \cP(\cX)$ is defined (as in \eqref{eq:covemb}) as the linear operator $\S_\bbP : \cH \to \cH$ characterized by
\begin{align*}
\langle f, \S_\bbP g\rangle_{\cH}
&= \int \langle f, k_{u}\rangle_{\cH} \langle g, k_{u}\rangle_{\cH} \, d\bbP(u)
\\& = \int f(u)g(u) ~d\bbP(u), \qquad \forall f, g \in \cH.
\end{align*}
This defines $\S_\bbP$ as a self-adjoint, positive semidefinite, and trace-class linear operator on $\cH$ \cite{bach2022information}:
$$
\S_{\bbP}^\ast = \S_{\bbP}, \qquad \S_{\bbP} \succeq 0, \qquad \|\S_{\bbP}\|_{\mathrm{tr}(\cH)} < \infty.
$$
Furthermore, when $\bbP$  has full support on $\cX$ and the bounded kernel $k$ is universal, the operator $\S_\bbP$ is injective on $\cH$.

\subsection*{Gaussian Measures}
\begin{figure}[!b]
  \centering
  \begin{subfigure}[b]{0.320\textwidth}
    \includegraphics[width=\linewidth, trim={0cm 3cm 0cm 3cm},clip]{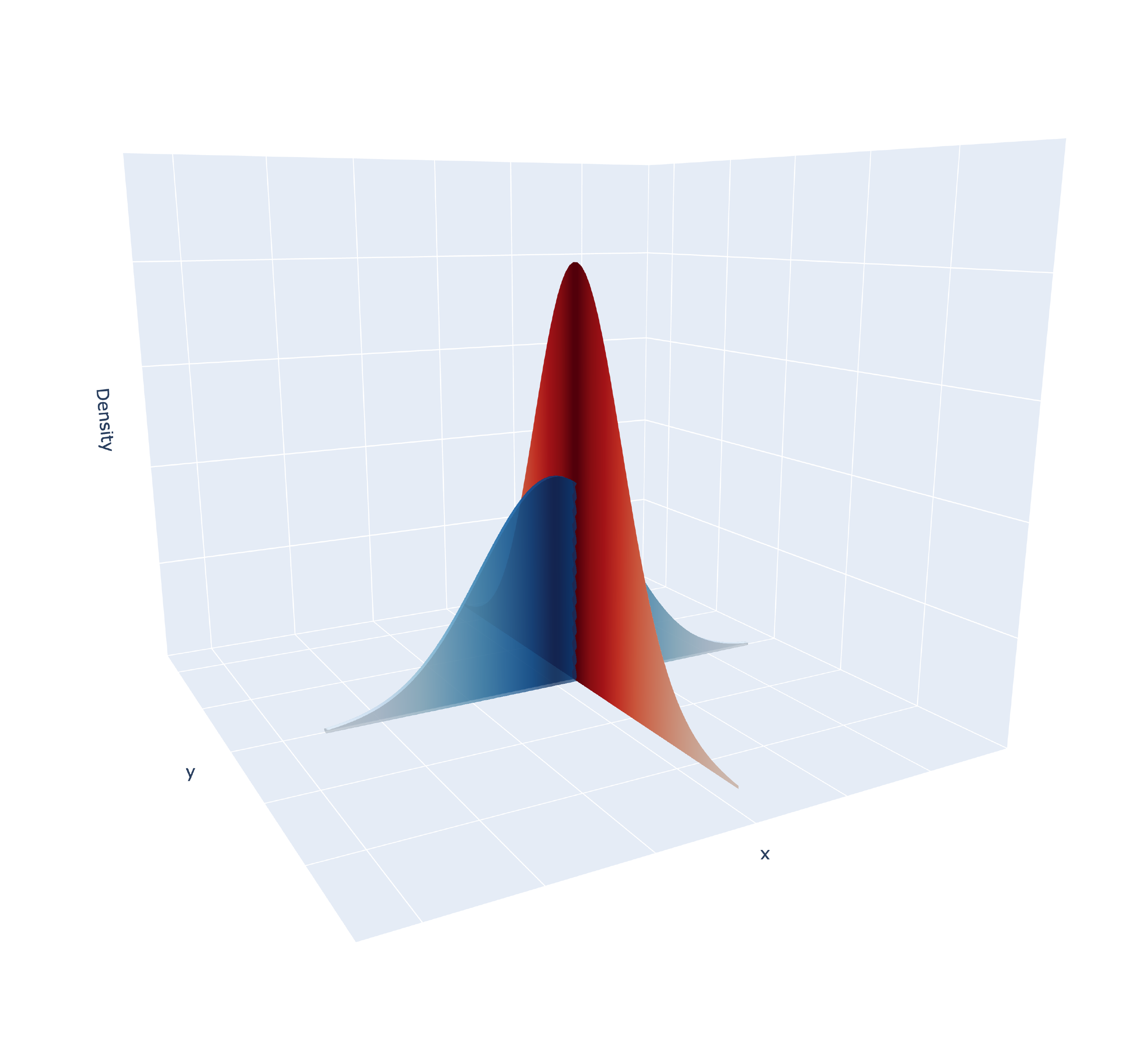}
  \end{subfigure}
  \begin{subfigure}[b]{0.320\textwidth}
\includegraphics[width=\linewidth, trim={0cm 3cm 0cm 3cm},clip]{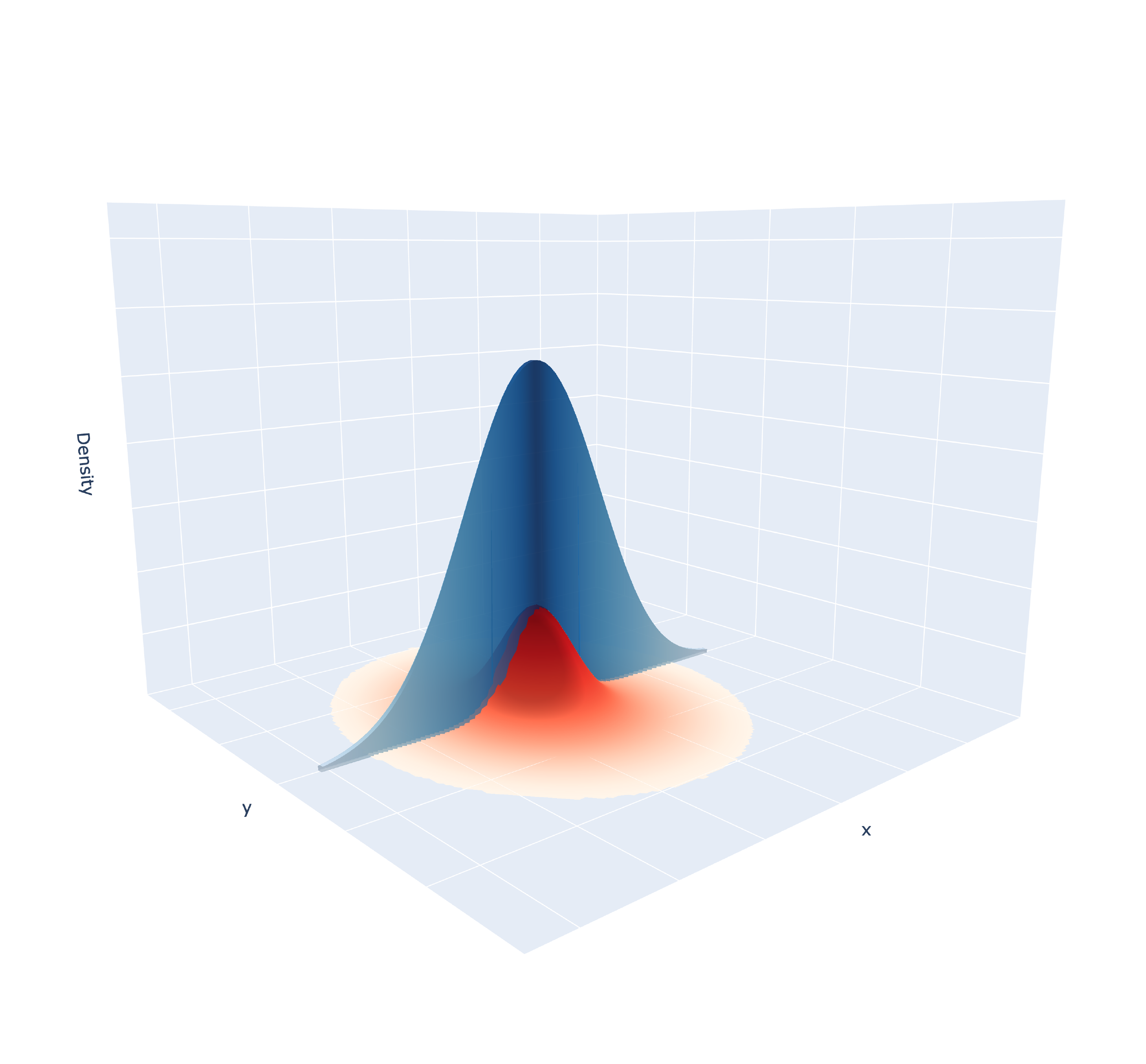}
  \end{subfigure}
  \begin{subfigure}[b]{.320\textwidth}
    \centering
    \includegraphics[width=\linewidth, trim={0cm 3cm 0cm 3cm},clip]{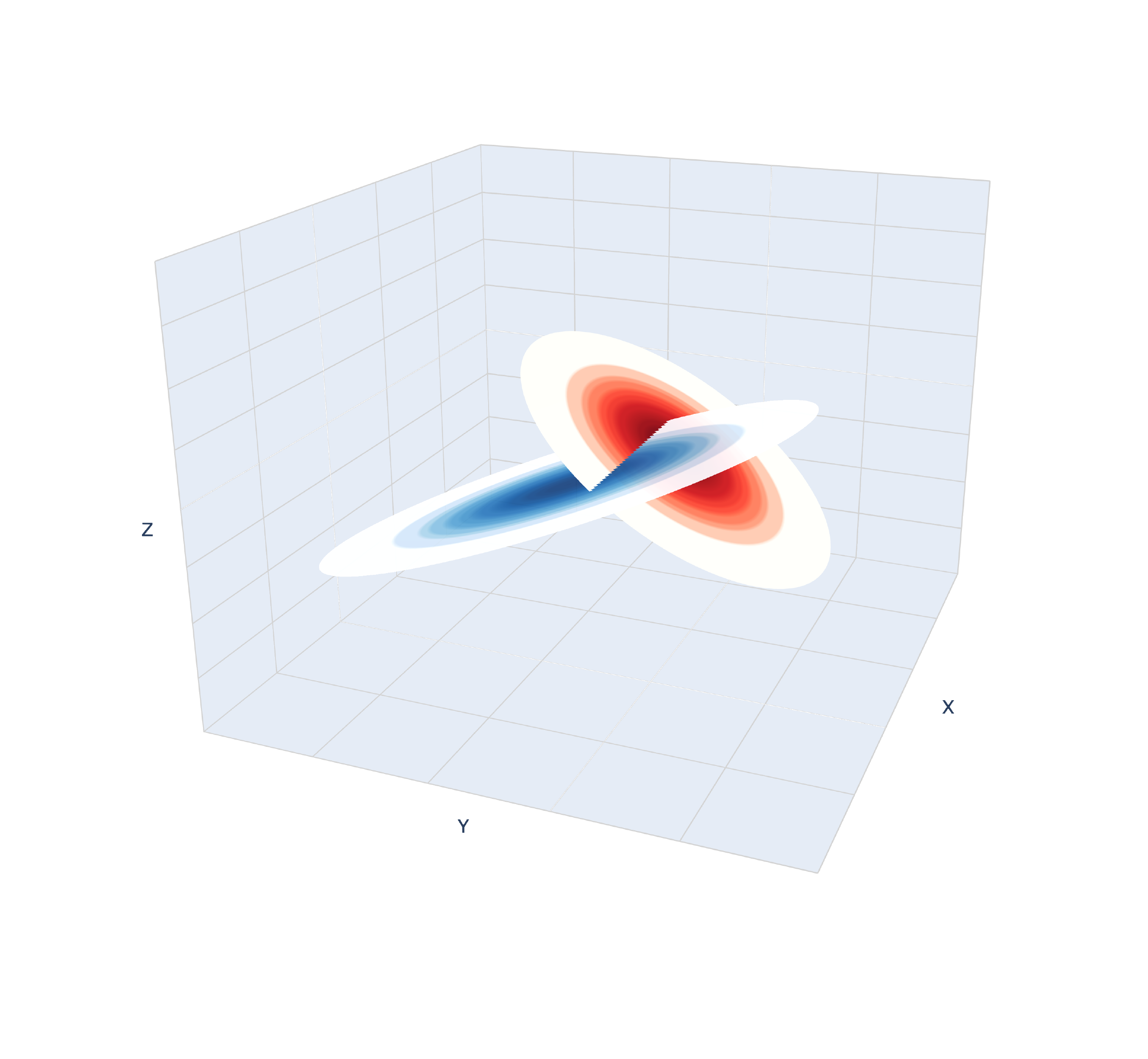}
  \end{subfigure}
  \caption{Since Gaussians are always supported on affine sets, there is structure to the way singularity can manifest.  In $\mathbb{R}^2$, for instance this can arise because the two Gaussians are supported on distinct lines (a) or because one is supported on the (full) plane, while the other on a line.  In  $\mathbb{R}^3$, mutual singularity of Gaussians can arise, for instance, when the measures are supported on distinct planes.}
\end{figure}
Recall that a Borel measure $\mu$ on a separable Hilbert space $\cH$ is Gaussian if and only if for a random element $\mathrm{X}$ with law $\mu$ and  for every $f \in \cH$, the scalar random variable 
$  \langle \mathrm{X}, f \rangle$ is a Gaussian random variable on $\bbR$. {A Gaussian measure $\mu$ on $\cH$ is determined by its mean vector $\m = \int \bu ~d \mu(\bu)\in\cH$ and covariance operator $\S=\int (\bu-\m)\otimes (\bu-\m)~d\mu(\bu)$. The latter is a self-adjoint, positive semidefinite, and trace-class operator $\cH\to\cH$}. Conversely, any $\cH$-vector and trace-class self-adjoint and positive semidefinite linear operator $\cH\to\cH$ give rise to a corresponding Gaussian, playing the respective role of mean and covariance. As with any two measures, two Gaussian measures $\mu, \nu$ on $\cH$  are said to be \textit{equivalent} (denoted $\mu \sim \nu$) if they have the same null sets:
$$
\mu(A) = 0 \iff \nu(A) = 0 \quad \text{for all measurable sets } A.
$$
This implies that the measures are absolutely continuous with respect to each other, the Radon-Nikodym derivatives $\frac{d\mu}{d\nu}$ and $\frac{d\nu}{d\mu}$ exist, and the two distributions share the same support. Two Gaussian measures are said to be \textit{singular} (denoted $\mu \perp \nu$) if for some measurable $A\subset \cH$
$$
\mu(A)= 0  \quad \text{and} \quad \nu(A^c) = 0.
$$
{This means that their supports are essentially disjoint (their intersection has measure zero under at least one of $\mu,\nu$), and there is no possible Radon-Nikodym derivative (density) of either with respect to the other.}

Given two Gaussian measures $\mu = \mathcal{N}(\m_{1}, \S_{1})$ and $\nu = \mathcal{N}(\m_{2}, \S_{2})$ on $\cH$, the \textit{Feldman-H\'{a}jek} Theorem \cite{hajek1958property,feldman1958equivalence} 
states that only two scenarios are possible:
\begin{align*}
    &\mu \sim \nu \qquad \text{(they are \textit{equivalent})}\\
    \qquad\mbox{or}\qquad
    &\mu \perp \nu \qquad  \text{(they are \textit{mutually singular}).}
\end{align*}
In particular, equivalence holds if and only if the following three conditions simultaneously hold true:
\begin{itemize}[itemsep=3pt, topsep=2pt]
    \item[(i)] They generate the same \emph{Cameron-Martin space}, i.e
        $\range(\S_{1}^{1/2}) = \range(\S_{2}^{1/2})=\range\left((\S_{1} + \S_{2})^{1/2}\right).$
                 
    \item[(ii)] The difference in their means lies in this common Cameron–Martin space, i.e.
        $
          \m_{1} - \m_{2} \in \range\left((\S_{1} + \S_{2})^{1/2}\right).
        $
    \item[(iii)] There exists a Hilbert-Schmidt operator $\bH$ with $\Id + \bH\succ 0$ such that $\S_1 = \S_2^{1/2}(\Id + \bH)\S_2^{1/2}$.
\end{itemize}

\noindent The Feldman-H\'{a}jek theorem is valid regardless of the dimensionality of the ambient Hilbert space $\cH$. But its most striking consequences manifest in the case of infinite dimensional $\cH$ (such as the RKHS of a universal kernel), where seemingly minute perturbations of a Gaussian vector can lead to singularity. For instance, let $X$ be a centered Gaussian vector of injective covariance on an infinite dimensional $\cH$. Then $\sqrt{1+\varepsilon}\, X$ and $X$ have singular laws for any $\varepsilon>0$: we have $\mathrm{cov}(\sqrt{1+\varepsilon}\,  X)=\mathrm{cov}^{1/2}(X)(\mathbf{I}+\varepsilon\mathbf{I})\mathrm{cov}^{1/2}(X)$, and $\mathbf{H}=\varepsilon \mathbf{I}$ fails to be Hilbert-Schmidt for any $\varepsilon>0$, no matter how small. Similarly, the laws of $X+\m$ and $X$ can become singular for an arbitrarily small translation $\|\m\|_{\cH}>0$, if $\m$ is not sufficiently smooth (in terms of the source condition $\m =\mathrm{cov}^{1/2}(X)g$ for some $g\in\cH$).  These are familiar phenomena in Gaussian diffusion modeling, where even a small mean shift and/or volatility change can lead to singularity of the induced path measures.

\section{Two-Sample Testing as Singular Gaussian Discrimination: A Separation of Measure Phenomenon}

\begin{figure}[h]
\centering
\begin{tikzpicture}
\node[draw, thick, inner sep=10pt, label=above:{$\mathcal{X}$ 
}] (box1) {
    \includegraphics[height=3cm, trim={2cm 2cm 2cm 2cm},clip]{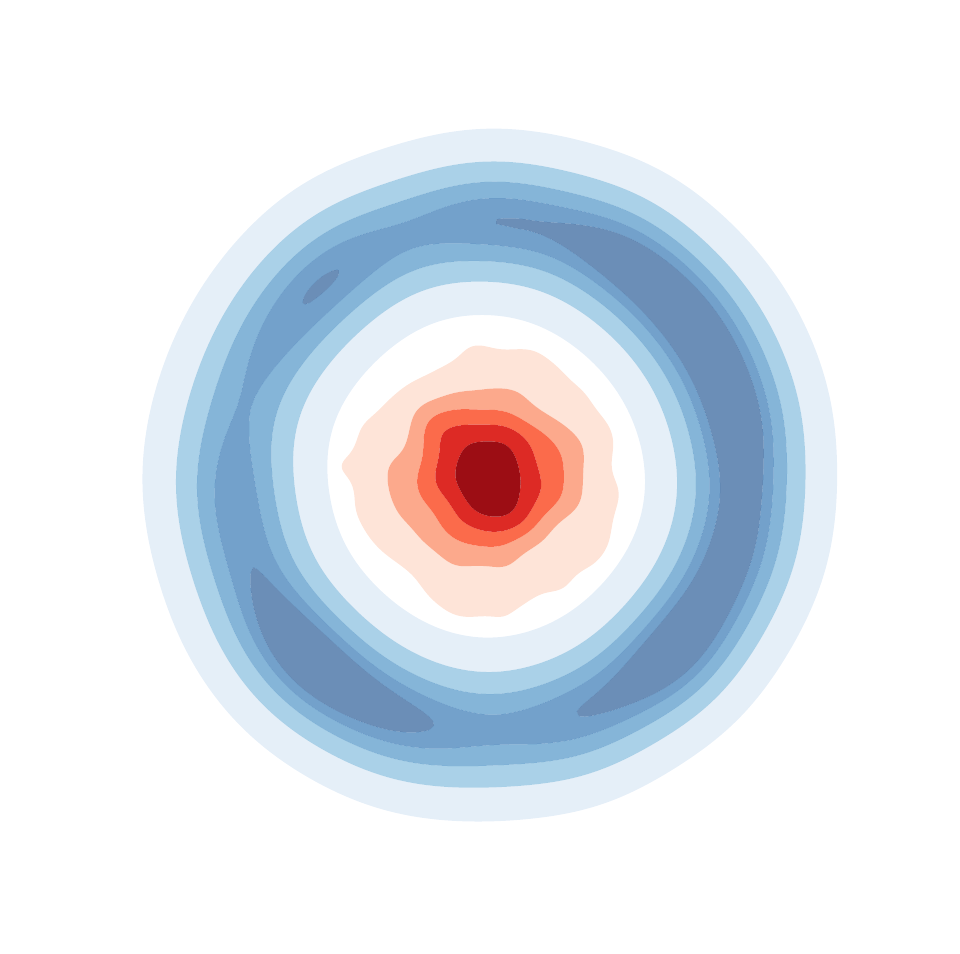}
};

\node[draw, thick, inner sep=10pt, right=.5cm of box1,
      label=above:{$\mathcal{H}$
      }] (box2) {
    \includegraphics[height=3cm, trim={8cm, 0 8cm 0},clip]{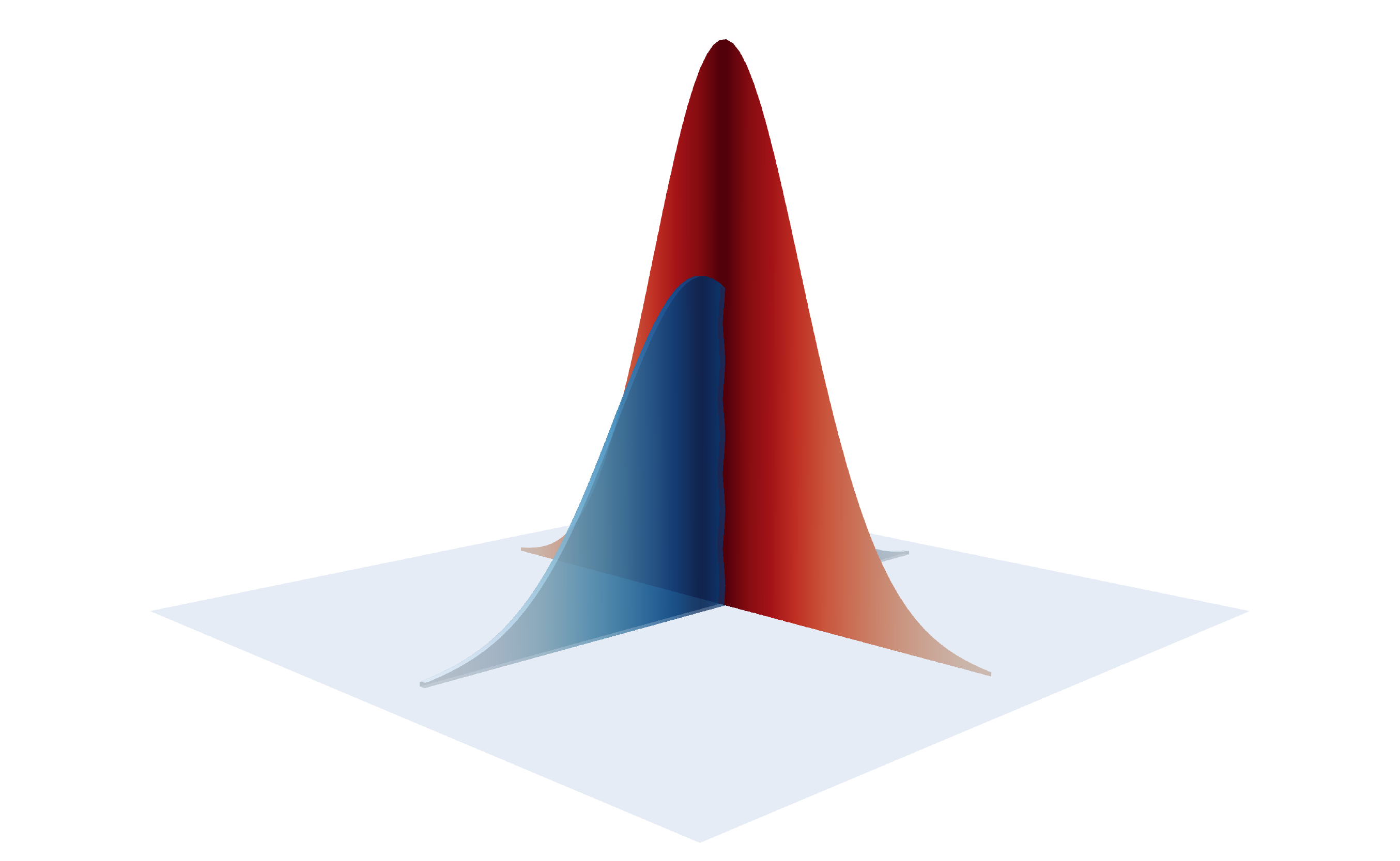}
};

\end{tikzpicture}
\caption{Gaussian embeddings \textit{magnify} distributional differences in a structured fashion: \textit{distinct} measures on $\mathcal{X}$ ($\mathbb{P},\mathbb{Q}$ on the left, whose contour lines are shown in red and blue respectively) are mapped to \textit{mutually singular} Gaussian measures on $\mathcal{H}$ ($\mathcal{N}_\mathbb{P}, \mathcal{N}_\mathbb{Q}$ on the right, in red and blue  respectively).}
\end{figure}

Consider the following (non-parametric) two-sample problem: given two (Borel) probability distributions $\bbP$ and $\bbQ$ on $\mathcal{X}$, we want to test the null hypothesis $H_0: \bbP = \bbQ$ against the alternative $H_1: \bbP \neq \bbQ$. {Apart from being non-atomic, the two probability measures can be arbitrary and need not satisfy any additional regularity conditions.}

In this section we state our main results, Theorem~\ref{thm:main} and Corollary~\ref{cor:main}. These state that the two-sample problem is equivalent to the problem of discriminating two singular Gaussian measures,  namely the two Gaussian measures corresponding to the embedding of $\bbP$ and $\bbQ$.
First, we consider the zero-mean Gaussian measures $\cN(\mathrm{0}, \S_{\bbP})$, $ \cN(\mathrm{0}, \S_{\bbQ})$ on $\cH$ with $\S_\bbP$, $\S_\bbQ$ as defined in \eqref{eq:covemb}: 
\begin{theorem}\label{thm:main}
    Let $\cX$ be a locally compact uncountable Polish space and $k: \cX \times \cX\to \bbR$ be a bounded $C_{0}(\cX)$-universal reproducing kernel thereon. If $\bbP,\bbQ$ are non-atomic (Borel) probability measures on $\cX$, then
$$
\bbP \neq \bbQ \quad\iff\quad \cN(\mathrm{0}, \S_{\bbP}) \perp \cN(\mathrm{0}, \S_{\bbQ}).
$$
\end{theorem}

{Of course, under the null hypothesis $H_0: \bbP = \bbQ$, the two embeddings are equal. The surprising aspect of the result is that, under the alternative, the two embeddings are not merely different, but vastly different from an information-theoretic perspective: they are mutually singular}. As a corollary to Theorem~\ref{thm:main}, we also obtain: 
\begin{corollary}\label{cor:main}
    Let $\cX$ be a locally compact uncountable Polish space and $k: \cX \times \cX\to \bbR$ be a bounded $C_{0}(\cX)$-universal reproducing kernel thereon. If $\bbP,\bbQ$ are non-atomic (Borel) probability measures on $\cX$ , then
    $$\bbP \neq \bbQ \quad\iff\quad \cN(\m_{\bbP}, \S_{\bbP}) \perp \cN(\m_{\bbQ}, \S_{\bbQ}).
    $$
\end{corollary}

These results illustrate a ``blessing of infinite-dimensionality": by suitably mapping into the space of Gaussian measures over an infinite-dimensional RKHS (of a universal kernel), we obtain a geometric representation that ``fully separates" the embedded measures. This considerably simplifies the task of distinguishing between distributions, reducing two-sample testing to testing for the \textit{essential disjointness} of the supports of Gaussians. Importantly, given samples from $\bbP$ and $\bbQ$, these Gaussian embeddings can be approximated by their empirical counterparts, uniformly in the dimension of $\mathcal{X}$. {From an information-theoretic perspective, the embedded Gaussians are ``infinitely separated" under the alternative regime: neither admits a density with respect to the other, and thus the Kullback-Leibler (KL) divergence of either with respect to the other is ill-defined (infinite). Nevertheless, the Feldman-H\'{a}jek criterion suggests that a projected KL divergence can be employed in order to operationalise the results via a quantitative version. Namely, one can consider a sequence of KL divergences, arising when the Gaussian measures are marginalised over a nested sequence of increasing subspaces, generated by an orthonormal basis.} 
Below, we write $T_\#\mu$ to denote the usual \emph{pushforward} of a Borel measure $\mu$ on $\mathcal{X}$ through a measurable map $T:\mathcal{X} \to \mathcal{X}$,  i.e. the Borel measure  defined by \((T_\#\mu)(B) := \mu\big(T^{-1}(B)\big)\) on Borel sets $B$. 
\begin{theorem}\label{thm:operat}
  Let $\bbP,\bbQ$ be non-atomic Borel probability measures on a locally compact uncountable Polish space $\cX$,  with $\bbQ \gg \bbP$, and let $k: \cX \times \cX\to \bbR$ be a bounded $C_0(\cX)$-universal reproducing kernel.
    Then:
    \begin{equation}
        \label{eq:operat}
        \lim_{N \to \infty}   
        D_{\textnormal{KL}}\left({\cP_N}_{\#} \cN_\bbQ \,||\,{\cP_N}_{\#}\cN_{\bbP}\right) 
        =
        \begin{cases}
            0, \: &\text{if} \quad \bbP=\bbQ,\\
            \infty, \: &\text{if} \quad \bbP\neq \bbQ.
        \end{cases}
    \end{equation}
    where $\cN_\bbP, \cN_\bbQ$ are either centered or uncentered Gaussian embeddings of $\bbP,\bbQ$, respectively, and $\cP_N = \sum_{i=1}^N e_i\otimes e_i$ is a sequence of projections with $\{e_i\}_{i\geq 1}$ comprising an orthonormal system of eigenvectors for $\S_\bbP$.
\end{theorem}

\begin{remark}
    The absolute continuity assumption $\bbQ\gg\bbP$ ensures finiteness of $D_{\textnormal{KL}}\left({\cP_N}_{\#} \cN_\bbQ \,||\,{\cP_N}_{\#}\cN_{\bbP}\right)$ for any finite truncation parameter $N$, and incurs no loss of generality in this context: one can always replace $\bbQ$ by the mixture $\bbQ'=\frac{1}{2}(\bbP +\bbQ)$, and observe that $\bbP = \bbQ \iff \bbP = \bbQ'$.
\end{remark}

The left-hand side of \eqref{eq:operat} can be understood as a regularized/truncated likelihood ratio between the two Gaussian embeddings \(\cN_\bbP\) and \(\cN_\bbQ\). Indeed,
 given measures $\mu,\nu$ such that the likelihood ratio \( \frac{d\mu}{d\nu}\) exists \(\nu\)-almost everywhere, we can express the KL divergence as
$$
D_{\mathrm{KL}}(\mu\,||\,\nu) = \int_{\cX} \log  \frac{d\mu}{d\nu} \, d\mu,
$$
i.e. as the expected log-likelihood ratio under $\mu$. In other words, the left hand side of \eqref{eq:operat} quantifies, on average, how much more (or less) likely a sample drawn from \(\cN_\bbQ\) is under \(\cN_\bbQ\) than under \(\cN_\bbP\), when viewed through its projection on a subspace of dimension $N$. Such projected likelihood ratios have a long history, and indeed their use in functional data analysis, as well as their potential for nearly perfect testing, is already identified by Grenander \cite{grenander1981abstract}. In classical two-sample tests, the power of the test depends continuously on the \emph{magnitude} of the difference between distributions. But here, the truncation parameter $N$ is \emph{user-controlled}, and represents a regularisation. Thus, for sufficiently large sample sizes (regulating the empirical approximation of the embedding) one can hope to obtain very powerful tests by proper choice of $N$. Other regularized versions of KL divergence can be formulated, and an in-depth study of such consistent and powerful tests operationalising the results herein presented is carried out in \cite{santoroLRT25}. They show that it is possible to  specify a proper balancing of sample size and regularisation and to implement tests enjoying both highly powerful empirical performance and rigorous asymptotic theoretical guarantees.

 The proofs of our main results are given in a separate section -- in fact, we provide two alternative proofs. We comment here on the two key properties on which they rely: (i) the fact that the embedded covariances can be related by suitable multiplication operators acting on functions over a continuous domain; and (ii) that such multiplication operators cannot be compact unless they are trivial (uniformly zero).

\section{The Roles of Mean Embedding vs Covariance Embedding}
A natural question  is whether the singularity result can be separately attributed to the kernel mean or the kernel covariance component of the embedding. The high-level answer is that separation is elicited by the  covariance component of the embedding, and only that. To make things precise, we revisit the Feldman-H\'{a}jek conditions: by separating condition (ii) from conditions (i) and (iii) in that statement, one can disentangle the respective roles played by the difference in means (a “translation” effect) and by the difference in covariance structure (a “multiplicative” effect). In particular, writing $\overline{\S} =\S_{(\bbP+\bbQ)/2}= (\S_{\bbP} + \S_{\bbQ})/2$ for the pooled covariance embedding, one has \cite{rao1963discrimination,shepp1966gaussian}:
\begin{equation}\label{eq:disintangled_FH}
       \begin{gathered}
\cN(\m_{\bbP}, \S_{\bbP}) \perp  \cN(\m_{\bbQ}, \S_{\bbQ}) 
\\
\textit{if and only if}
\\
\Big\{\quad\cN(\m_{\bbP}, \overline{\S})\perp \cN(\m_{\bbQ}, \overline{\S})   \quad\mbox{or}\quad
\cN(\mathrm{0}, \S_{\bbP})\perp\cN(\mathrm{0}, \S_{\bbQ})  \quad \Big\}.
\end{gathered}
\end{equation}
Theorem~\ref{thm:main} shows that $\bbP\neq\bbQ$ if and only if $\cN(\mathrm{0}, \S_{\bbP})\perp \cN(\mathrm{0}, \S_{\bbQ})$, without any reference to the mean embedding. Therefore, we know that the covariance embedding alone
will always guarantee singularity under the alternative.

It remains to consider whether the mean component of the embedding alone could also guarantee singularity under the alternative -- in other words whether $\bbP\neq\bbQ$ if and only if $\cN(\m_{\bbP}, \overline{\S})\perp \cN(\m_{\bbQ}, \overline{\S})$. The answer is no, in a rather strong sense --- the following proposition establishes that  $\cN(\m_{\bbP}, \overline{\S})$ and $\cN(\m_{\bbQ}, \overline{\S})$ will \emph{never} be singular:
\begin{proposition}\label{prop:main:mean}
    Let $\bbP,\bbQ$ be non-atomic Borel probability measures on a locally compact uncountable Polish space $\cX$, and $k: \cX \times \cX\to \bbR$ be a bounded $C_0(\mathcal{X})$-universal  reproducing kernel.  
    The Gaussian measures $\cN(\m_{\bbP}, \overline{\S})$ and $\cN(\m_{\bbQ}, \overline{\S})$ on $\cH$  are always mutually equivalent,  regardless of whether $\bbP=\bbQ$ or $\bbP\neq\bbQ$. 
    \end{proposition}

In principle, thus, discrimination criteria solely targeting the mean embedding shift (whether or not the latter is whitened by the mixture's covariance embedding) provide a weaker measure of discrimination, as they are blind to the separation of measure induced by the covariance embedding shift.

\section{Discussion}

Our theoretical results in the last two sections suggest that leveraging the separation-of-measure phenomenon through the Gaussian embedding promises substantial gains in statistical efficiency, as compared to methods focusing purely on mean embeddings. Of course, any  procedure specifically targeting the induced Gaussian singularity ultimately remains a kernel embedding procedure. Consequently, the user will still benefit from the same computational simplifications afforded by the ``kernel trick" \cite{Scholkopf2002LearningWithKernels}, and be confronted with the same choices omnipresent in kernel methods \cite{muandet2017kernel}, including the choice of kernel family and the choice of concentration parameter (bandwidth) \cite{steinwart2001influence,chapelle2002choosing}. The principles guiding these choices are task-dependent, but remain the same in the context of exploiting separation of measure.

\begin{figure*}[htbp]
        \centering
        \includegraphics[width=.9\linewidth]{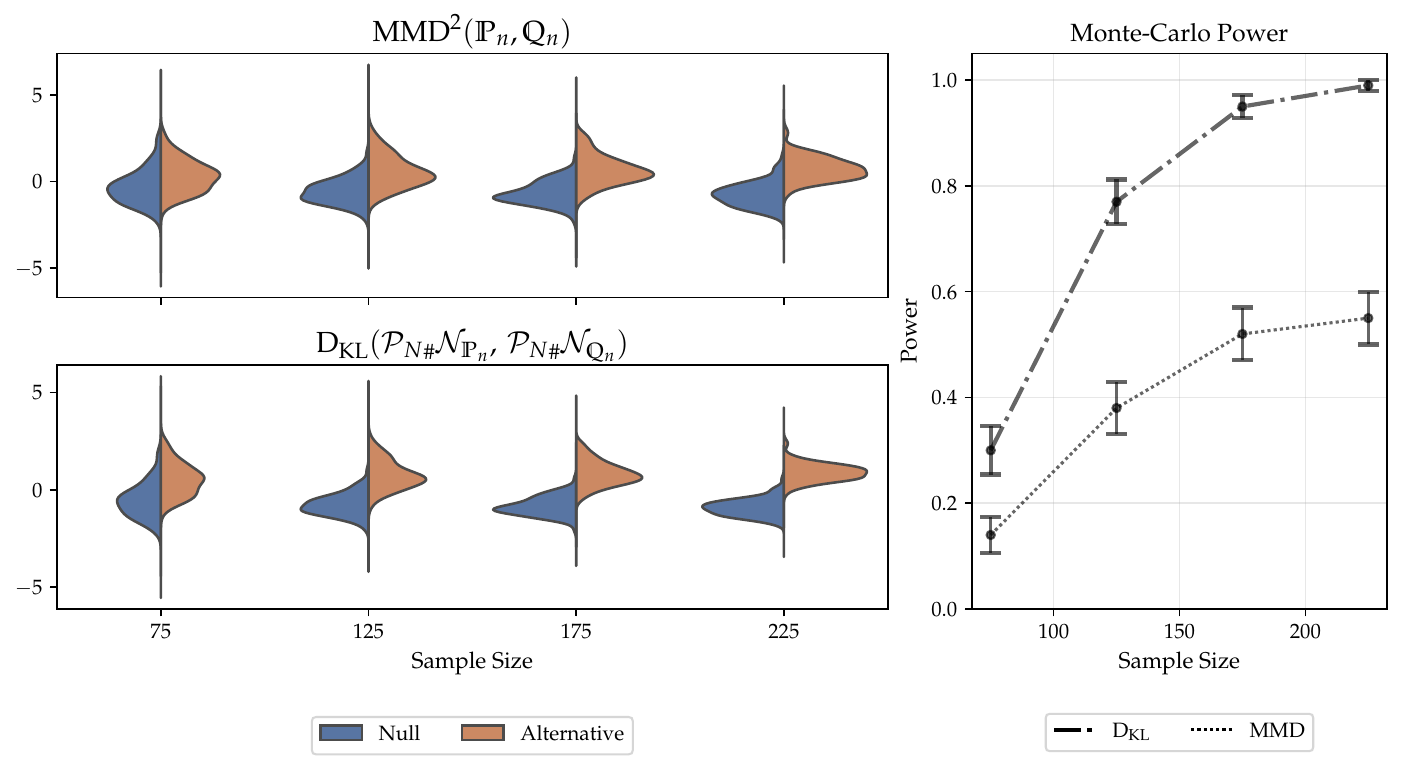}        
        \caption{Monte Carlo Illustration of the sampling behaviour of MMD and $\mathsf{D}_{\mathsf{KL}}$ under null  ($\alpha=1/2,\varepsilon=0$) and alternative ($\alpha=1/2,\varepsilon=1/4$) regimes, using a Laplacian kernel calibrated by median heuristic. \textbf{Left:} Smoothed Monte Carlo sampling distributions under $H_0$ and $H_1$, centered for ease of comparison. \textbf{Right:} Proportion of realisations under $H_1$ where each statistic exceeds the 95th percentile of the null sampling distribution based on $K = 100$ runs. }
        \label{fig:power}
\end{figure*}

For instance, in the context of testing, the kernel type should ideally reflect feature maps that are well-adapted to capture anticipated deviations from the null. That being said, Gaussian or Laplacian kernels are widely chosen in this context because they are known to be universal/characteristic on $\mathbb{R}^d$ \cite{sriperumbudur2011universality}, as required for consistent testing (unlike polynomial kernels, whose finite-dimensional RKHS precludes universality/characteristicness). Therefore, it is bandwidth selection that is the more dominant consideration, reflecting the scale at which differences are sought. Popular choices for empirically choosing a bandwidth include the median heuristic \cite{fukumizu2009kernel}, multiscale-kernels, and adaptive bandwidths \cite{gretton2012optimal,schrab2023mmd}. The same considerations apply when employing a separation-informed criterion such as \ref{eq:operat}, and underpin the choice of kernel type and bandwidth. But, within the same setting, and for the same type of kernel and method of bandwidth selection, we expect higher power when we make use of a separation-informed criterion such as \ref{eq:operat} (compared to a criterion targeting the mean embedding alone).

The following toy example provides a numerical sketch of this point in the standard setting $\mathcal{X}=\mathbb{R}^d$.   
Consider empirical distributions $
\bbP_n =\frac{1}{n}\sum_{i=1}^n \delta_{X_i},
$ and $
\bbQ_n =\frac{1}{n}\sum_{i=1}^n \delta_{Y_i}
$ sampled from centered $d$-dimensional Gaussian laws,   
\begin{equation*}
\begin{aligned}
X_1,\ldots,X_n &\stackrel{i.i.d.}{\sim} \bbP \equiv \mathcal{N}(0,C_{\alpha}) \\
Y_1,\ldots,Y_n &\stackrel{i.i.d.}{\sim} \bbQ \equiv \mathcal{N}(0,C_{\alpha+\varepsilon})
\end{aligned}
\end{equation*}
with AR(1) covariance structure $C_{\rho}=\{\rho^{|i-j|}\}_{i,j=1}^{d}$, for a correlation coefficient $\rho\in (0,1)$. In this context, we can contrast the null ($\varepsilon=0$) and alternative ($\varepsilon>0$) sampling distributions of:
\begin{itemize}

\item a  test statistic based purely on mean embedding, e.g. the MMD statistic $\|\m_{\bbP_n}-\m_{\bbQ_n}\|_{\cH}$, serving as a benchmark,

\item the test statistic $D_{\textnormal{KL}}\left({\cP_N}_{\#} \cN_{\bbQ_m} \,||\,{\cP_N}_{\#}\cN_{\bbP_n}\right)$ described in Equation  \ref{eq:operat}, targeting the separation phenomenon,
\end{itemize}
when both statistics use the same kernel type (taken as Laplace) and the same bandwidth selection procedure (taken as the median heuristic).  Figure~\ref{fig:power} probes these sampling distributions and illustrates the potential efficiency gains for ambient dimension $d=50$, sample sizes $n\in\{75,125,175,225\}$ and projection dimension $N=n$, based on one hundred Monte Carlo runs in each case. One can appreciate how the overlap between null  ($\alpha=1/2,\varepsilon=0$) and alternative ($\alpha=1/2,\varepsilon=1/4$) sampling distributions decreases more rapidly in $n$ for the separation-informed statistic. This is more finely highlighted by the Monte-Carlo power curves, showing the proportion of samples under the alternative where each statistic exceeds the 95th percentile of its corresponding null sampling distribution. Though power increases with sample size for both statistics, the separation-of-measure phenomenon manifests clearly in the form of a clear domination relation between the two curves.

\section{Proofs of Main Results}

Given a  Borel measure $\bbGamma$ on $\cX$, denote by $\bJ_{\bbGamma}$  the embedding:
\begin{equation}\label{eq:embedding_HtoL2}
\bJ_{\bbGamma} \::\: \cH \to L^{2}(\cX,\bbGamma), \qquad  f\mapsto f,
\end{equation}
where $L^{2}(\cX,\bbGamma)$ is the space of square-integrable functions on $\cX$ with respect to the measure $\bbGamma$.  Since $\mathcal X$ is a locally compact Polish space,  $C_{0}(\cX)$-universality of $k$ implies that it is also $L^{2}(\cX, \bbGamma)$-universal for any Borel measure $\bbGamma$ \cite[Theorem 4.1 and Corollary 4.4]{Carmeli2010}. In turn, $L^2(\cX,\bbGamma)$-universality of $k$ implies that the image of $\bJ_{\bbGamma}$ (i.e., the set of RKHS functions) is dense in $L^2(\cX, \bbGamma)$. When $\bbGamma$ is some measure dominating $\bbP$ and $\bbQ$ ($\bbP,\bbQ\ll \bbGamma$), for example $\bbGamma = \frac{1}{2}\bbP + \frac{1}{2}\bbQ$, we denote the respective densities as $d\bbP/d\bbGamma$ and $d\bbQ/d\bbGamma$. These densities allow to relate the inner products of the corresponding $L^2$ spaces via the relation
 $
\langle f,g\rangle_{L^2(\bbP)}=\int fg\,d\bbP=\int f\frac{d\bbP}{d\bbGamma}g\,d\bbGamma=\langle f, \bM_{{d\bbP}/{d\bbGamma}} g\rangle_{L^2(\bbGamma)} 
$, 
where the \emph{multiplication operator} $\bM_{{d\bbP}/{d\bbGamma}}(g)=\frac{d\bbP}{d\bbGamma}g$ encodes the change of measure from $\bbGamma$ to $\bbP$ as a pointwise  reweighting. The roles of the embedding operator $\bJ_{\bbGamma}$ and the multiplication operator $\bM_{d\bbP/d\bbGamma}$ can be understood as consequences of viewing the same function under different ambient Hilbert geometries. Elements of the RKHS $\mathcal H$ are functions on $\mathcal X$, but their norm encodes kernel-quantified smoothness rather than pointwise  magnitude. The embedding $\bJ_{\bbGamma}:\mathcal H\to L^2(\mathcal X,\bbGamma)$ simply regards an RKHS element as an ordinary square-integrable function, retaining pointwise values ($\bbGamma$-almost everywhere)  while discarding the RKHS geometry.   Similarly, when $\bbP\ll\bbGamma$,  embedding  $L^2(\mathcal X,\bbP)$ into $L^2(\mathcal X,\bbGamma)$ effectively rescales functions by $\sqrt{d\bbP/d\bbGamma}$, as reflected by the appearance of the multiplication operator $\bM_{d\bbP/d\bbGamma}$ in the change of inner product. The following lemma shows how $\bJ_\bbGamma$ and $\bM_{d\bbP/d\bbGamma}$ ``interlace" to yield a representation of $\S_{\bbP}$ that goes through the $L^2(\mathcal{X},\bbGamma)$ geometry, when $\bbP$ has essentially bounded density w.r.t. $\bbGamma$: 
\begin{lemma}\label{lem:Semb_is_multop}
    Let $\bbP,\bbGamma \in \cP(\mathcal{X})$ with $\bbP\ll \bbGamma$ and $\|d\bbP/d\bbGamma\|_{L^\infty(\cX,\bbGamma)}<\infty$. Then we can decompose
    $
    \S_{\bbP} = \bJ_{\bbGamma}^{\ast} \bM_{d\bbP/d\bbGamma} \bJ_{\bbGamma}.
   $
    where $\bJ_{\bbGamma}$ denotes the embedding operator  of $\cH$ into $L^2(\cX,\bbGamma)$ and  $\bM_{d\bbP/d\bbGamma}$ denotes the multiplication operator,
    \begin{equation}\label{eq:multop}
    \begin{split}
        \bM_{d\bbP/d\bbGamma}: L^{2}(\mathcal{X},\bbGamma) \to L^{2}(\mathcal{X},\bbGamma)
        \\ (\bM_{d\bbP/d\bbGamma}g)(x) = \frac{d\bbP}{d\bbGamma}(x)g(x).
    \end{split}
    \end{equation}
     
\end{lemma}

\begin{proof}[Proof of Lemma~\ref{lem:Semb_is_multop}]
    Notice that the range $\range(\bJ_{\bbGamma})$ of $\bJ_{\bbGamma}$ is dense if $k$ is universal.
    Note that for any $f \in L^{\infty}(\cX, \bbGamma)$, the multiplication operator $\bM_{f}$ is a bounded operator, and $\|\bM_{f}\| = \|f\|_{L^\infty(\mathcal{X},\bbGamma)}$.
    For $f,g\in \cH$, we have:
    \begin{align*}
        \langle f, \S_\bbP g\rangle_{\cH}
        &= \int_{\cX} \langle f, k_\bu\rangle \langle g, k_\bu\rangle \, ~d\bbP(\bu)
        \\&= \int_{\cX} f(\bu)g(\bu)\frac{d\bbP}{d\bbGamma}(\bu)  ~d\bbGamma(\bu) 
        \\&= \langle \bJ_{\bbGamma} f, \bM_{d\bbP/d\bbGamma}\bJ_{\bbGamma} g\rangle_{L^2(\cX,\bbGamma)}
        \\&= \langle f, \bJ_{\bbGamma}^{\ast}\bM_{d\bbP/d\bbGamma}\bJ_{\bbGamma} g\rangle_{\cH}
    \end{align*}
    where $\bM_{d\bbP/d\bbGamma}$ is the multiplication operator on $L^2(\cX,\bbGamma)$ corresponding to the density of $\bbP$ with respect to $\bbGamma$.
    \end{proof}

The second ingredient required for the proof of Theorem~\ref{thm:main} is the observation that the multiplication operator in the previous lemma cannot be compact. This stems from  a classical result (e.g. \cite[][Corollary 1.1]{Singh1979}) which we state and prove here in our precise context and notation, for the sake of completeness:
\begin{lemma}\label{lem:multisnotcompact}
    Let $\bbGamma$ be a non-atomic Borel probability measure on $\cX$, and consider the space $L^2(\cX,\bbGamma)$ of square-integrable functions with respect to $\bbGamma$. For $f\in L^2(\cX,\bbGamma)$, let $\bM_f$ be the multiplication operator defined in \eqref{eq:multop}. Then $\bM_f$ is compact if and only if $f=0$ $\bbGamma$-almost everywhere.
\end{lemma}
\begin{proof}[Proof of Lemma~\ref{lem:multisnotcompact}]
    Suppose that $f \ne 0$ on a set of positive measure. Then there exists $\delta > 0$ such that the set
    \(
    A := \{x \in \cX : |f(x)| > \delta \}
    \)
    has positive $\bbGamma$-measure. Consider the orthonormal system $(e_n)$ in $L^2(\cX, \bbGamma)$ defined by the scaled indicator functions 
    \(
    e_n = {\mathbf{1}_{A_n}}/{\sqrt{\bbGamma(A_n)}} 
    \), 
    where $(A_n)$ is a sequence of disjoint measurable subsets of $A$ with positive and finite measure, e.g., obtained by partitioning $A$ into countably many pieces.
    Each $e_n$ is supported on $A_n \subset A$, so $|f| > \delta$ on $A_n$, and thus
    \begin{align*}
    \| \bM_f e_n \|_{L^2(\cal{X},\bbGamma)}^2 
    &= \int_{\cX} |f(x)|^2 |e_n(x)|^2 \, d\bbGamma(x)
    \\ &= \int_{A_n} |f(x)|^2 \cdot \frac{1}{\bbGamma(A_n)} \, d\bbGamma(x)
    \ge \delta^2.
    \end{align*}
    So $\| \bM_f e_n \| \ge \delta$ for all $n$.
    Moreover, since the $e_n$ are orthonormal, the sequence $(\bM_f e_n)$ has no convergent subsequence (in norm). Therefore, the image of the unit ball under $\bM_f$ is not relatively compact, and hence $\bM_f$ is not compact.
    
\end{proof}
We can now give the proofs of our main results.
\begin{proof}[Proof of Theorem~\ref{thm:main}]
 \textbf{Step 1.} Assuming equivalence of $\cN(\mathrm{0}, \S_{\bbP})$ and $\cN(\mathrm{0}, \S_{\bbQ})$, by the Feldman-H\'{a}jek Theorem \cite[see][Corollary 6.4.11]{bogachev1998gaussian}, there exists some Hilbert-Schmidt operator $\bH$ with eigenvalues greater than $-1$ such that:
\begin{equation}\label{eq:equiv}
    \S_{\bbQ} = \sqrt{\S_{\bbP}} (\bI + \bH) \sqrt{\S_{\bbP}}.
\end{equation} Let $\bbGamma = \tfrac{1}{2}(\bbP + \bbQ)$, and $p=d\bbP/d\bbGamma$, $q=d\bbQ/d\bbGamma$ be the corresponding densities. Observe that $\|p\|_{L^\infty(\mathcal{X},\bbGamma)}\leq \|p+q\|_{L^\infty(\mathcal{X},\bbGamma)}=\|d(2\bbGamma)/d\bbGamma\|_{L^\infty(\mathcal{X},\bbGamma)}=2$.
Therefore, by Lemma~\ref{lem:Semb_is_multop} we have that $\S_{\bbP} = (\bM_{\sqrt{p}}\bJ_{\bbGamma})^{\ast}(\bM_{\sqrt{p}}\bJ_{\bbGamma})$.
Recall that $\bU: \cH \to L^{2}(\cX, \bbGamma)$ is a \emph{partial isometry} if $\bU^{\ast}\bU$ (or equivalently, $\bU\bU^{\ast}$) is a projection operator. 
By polar decomposition \cite[see][Theorem 2.4.8]{Simon2015}, we can write $\bM_{\sqrt{p}}\bJ_{\bbGamma} = \bU\sqrt{\S_{\bbP}}$ for some partial isometry $\bU$ such that $\ker(\bU) = \ker(\S_{\bbP})$. By self-adjointness, $\sqrt{\S_{\bbP}} = \bU^{\ast}\bM_{\sqrt{p}}\bJ_{\bbGamma} = \bJ_{\bbGamma}^{\ast}\bM_{\sqrt{p}}\bU$. By \eqref{eq:equiv}:
\begin{align*}
    \bJ_{\bbGamma}^{\ast} \bM_{q} \bJ_{\bbGamma}
    &= \sqrt{\S_{\bbP}} (\bI + \bH) \sqrt{\S_{\bbP}}  
    \\& = \bJ_{\bbGamma}^{\ast}\bM_{\sqrt{p}}\bU\bU^{\ast}\bM_{\sqrt{p}}\bJ_{\bbGamma}  +  \bJ_{\bbGamma}^{\ast}\bM_{\sqrt{p}}\bU\bH\bU^{\ast}\bM_{\sqrt{p}}\bJ_{\bbGamma}
\end{align*}
which in turn readily implies that:
$$
\bJ_{\bbGamma}^{\ast} \left[\bM_{q} - \bM_{\sqrt{p}}\bU\bU^{\ast}\bM_{\sqrt{p}} - \bM_{\sqrt{p}}\bU\bH\bU^{\ast}\bM_{\sqrt{p}} \right] \bJ_{\bbGamma} = \bzero.
$$
Since $k$ is $C_{0}(\cX)$-universal, it is also $L^{2}(\cX, \bbGamma)$-universal (as $\cX$ is locally compact, second-countable and Hausdorff \cite[see][Theorem 4.1 and Corollary 4.4]{Carmeli2010} and the image of $\bJ_{\bbGamma}$ is dense. Therefore, $\bM_{q} = \bM_{\sqrt{p}}\bU\bU^{\ast}\bM_{\sqrt{p}} + \bM_{\sqrt{p}}\bU\bH\bU^{\ast}\bM_{\sqrt{p}}$ which implies that $\mathrm{supp~} q \subset \mathrm{supp~} p$.
By symmetry, it must be $\mathrm{supp~} p = \mathrm{supp~} q \,(=: S$, say).  Define 
\begin{equation}\label{eqn:interpret}
        (q/p)(x) := \begin{cases}
            q(x)/p(x) &\mbox{for } x\in S\\
            0 &\mbox{otherwise.}
        \end{cases}
\end{equation}
\textbf{Step 2.} Consider the subspace $M$ of functions $f$ such that $\mathrm{supp\,} f \subset S$.
Notice that for every $f \in M$, there exists a sequence $\{f_{j}\}_{j=1}^{\infty}$ such that $\sqrt{p}f_{j} \to f$ as $j \to \infty$. Then for $f \in M$,
    \begin{align*}
        \langle &f, \bM_{q/p}f \rangle =
        \\&= \lim_{j \to \infty} ~\langle \sqrt{p}f_{j}, \bM_{q/p} \sqrt{p}f_{j} \rangle 
        \\&= \lim_{j \to \infty} ~\langle f_{j}, \bM_{q}f_{j} \rangle
        = \lim_{j \to \infty} ~\langle f_{j}, [\bM_{\sqrt{p}}\bU\bU^{\ast}\bM_{\sqrt{p}} + \bM_{\sqrt{p}}\bU\bH\bU^{\ast}\bM_{\sqrt{p}}]f_{j} \rangle
        \\& = \lim_{j \to \infty} ~\langle \sqrt{p}f_{j}, [\bU\bU^{\ast} + \bU\bH\bU^{\ast}]\sqrt{p}f_{j} \rangle 
        \\&= \langle f, [\bU\bU^{\ast} + \bU\bH\bU^{\ast}]f \rangle.
    \end{align*}
Using the fact that $L^{2}(\cX, \bbGamma) = M \oplus M^{\perp}$, it follows that $\bM_{q/p} = \Pi_{M}^{\ast}[\bU\bU^{\ast} + \bU\bH\bU^{\ast}]\Pi_{M},$ where $\Pi_{M}$ is the projection in $L^{2}(\cX, \bbGamma)$ to $M$. The right-hand side is bounded since $\|\Pi_{M}^{\ast}[\bU\bU^{\ast} + \bU\bH\bU^{\ast}]\Pi_{M}\|_{\text{op}(\cH)} \leq 1 + \|\bH\|_{\text{op}(\cH)} \leq 1 + \|\bH\|_{\text{HS}(\cH)} < \infty$ implying $\|q/p\|_{L^{\infty}(\cX, \bbGamma)} = \|\bM_{q/p}\|_{\text{op}(L^2(\cX,\bbGamma))} < \infty$.\\ 

\noindent
\textbf{Step 3.} Notice that $\Pi_{M}^{\ast}\bU\bU^{\ast}\Pi_{M}$ is a projection and $\Pi_{M}^{\ast}\bU\bH\bU^{\ast}\Pi_{M}$ is compact because $\bH$ is compact. 
Hence:
\begin{align*}
     &\bM_{(q/p) - 1} = \Pi_{M}^{\ast}\bU\bH\bU^{\ast}\Pi_{M}  
      \quad\textnormal{on}\quad
     \mathcal{R}(\Pi_{M}^{\ast}\bU\bU^{\ast}\Pi_{M}) \quad    \textnormal{and}  \\
    &\bM_{q/p} = \Pi_{M}^{\ast}\bU\bH\bU^{\ast}\Pi_{M}
      \quad\textnormal{on}\quad
     \mathcal{R}(\bI - \Pi_{M}^{\ast}\bU\bU^{\ast}\Pi_{M}).
\end{align*}
Either $\dim \mathcal{R}(\Pi_{M}^{\ast}\bU\bU^{\ast}\Pi_{M}) = \infty$ or $\dim \mathcal{R}(\bI - \Pi_{M}^{\ast}\bU\bU^{\ast}\Pi_{M}) = \infty$. As a consequence, at least one of $\bM_{(q/p) - 1}$ or $\bM_{q/p}$ has to be zero,  since, by Lemma \ref{lem:multisnotcompact}, there are no compact nonzero multiplication operators on $L^{2}(\cX, \bbGamma)$ (or its infinite-dimensional subspaces such as $M$) when $\bbGamma$ is non-atomic. This implies that either $p = q$ or $q = 0$ with the latter being impossible since $q$ is a probability measure. It follows that $\bbP = \bbQ$. The converse is immediate.
\end{proof}

\noindent A less abstract proof based on a coordinate-wise argument is possible if we assume, in addition, that $k$ is continuous:

\begin{proof}[Alternative coordinate-wise proof for continuous $k$] One direction remains immediate. For the non-trivial direction, let $\bbGamma = \tfrac{1}{2}(\bbP + \bbQ)$ and $\bS_{\bbGamma} = \int k_{x} \otimes k_{x} \, d\bbGamma(x)$, noting that $\bS_{\bbGamma} = \bJ_{\bbGamma}^{\ast}\bJ_{\bbGamma}^{\phantom{\ast}}$. 
    We first prove $\cN(0,\S_\bbQ)\sim\cN(0,\S_\bbP) \Rightarrow \bbP=\bbQ$ assuming $\mathrm{supp}(\bbGamma)=\mathcal{X}$, and then show how this reduction establishes the general case.
    
\smallskip

    \noindent\textbf{Full support}. Assume $\supp\{\bbGamma\}=\mathcal{X}$. Since $k$ is $C_0(\mathcal{X})$-universal, any $f\in\mathcal{H}$ is continuous on $\mathcal{X}\equiv\supp\{\bbGamma\}$. Hence
\(
f=0 \iff 
0=\int_{\mathcal{X}} |f(x)|^2\, d\bbGamma(x)
\equiv \langle f,\S_\bbGamma f\rangle_{\mathcal H}
\iff f\in \ker(\S_\bbGamma^{1/2}).
\)
We conclude that $\ker(\S_\bbGamma^{1/2})=\{0\}$.
    Now consider the operator $\mathbf{H} :=\S_{\bbGamma}^{-1/2}(\S_{\bbGamma} - \S_{\bbQ})\S_{\bbGamma}^{-1/2}$, which is well defined on $\mathcal{R}(\S_\bbGamma^{1/2})$. 
    Note that $\cN(0,\S_\bbQ)\sim\cN(0,\S_\bbP) \Rightarrow \cN(0,\S_\bbQ)\sim\cN(0,\S_\bbGamma)$;
    hence, by the Feldman-H\'{a}jek Theorem 
    \cite[see][Corollary 6.4.11]{bogachev1998gaussian}, the operator 
     $\mathbf{H}$  extends to a bounded (indeed Hilbert-Schmidt) operator on $\overline{\mathcal{R}(\S_\bbGamma^{1/2})}=\ker^\perp(\S_\bbGamma^{1/2})=\mathcal{H}$.
     We will now show this implies $\bbP=\bbQ$.

    \smallskip

    Let $\{(\gamma_{k}, \phi_{k})\}_{k\geq 1}$ be eigenvalues and eigenfunctions of  $\S_{\bbGamma}$, and notice that $\{\phi_{k}\}_{k \geq 1}$ forms a complete orthonormal basis of $\cH$. Furthermore, since $C_{0}(\cX)$-universality implies $L^{2}(\cX, \bbGamma)$-universality \cite[see][Theorem 4.1]{Carmeli2010}) we notice that, defining $\{f_{j}\}_{j\geq 1} \subset L^2(\cX, \bbGamma)$ as $f_{j} := \frac{1}{\sqrt{\gamma_{j}}}\bJ_{\bbGamma}\phi_{j}$, the system
    $\{f_{j}\}_{j\geq 1}$ forms an orthonormal basis of $L^2(\cX, \bbGamma)$, 
    \begin{equation}
    \label{eq:fjisbasis}
    \begin{split}
        \langle f_{j}, f_{k} \rangle_{L^{2}(\cX, \bbGamma)} 
        &= \frac{1}{\sqrt{\gamma_{j}\gamma_{k}}} \langle \bJ_{\bbGamma}^{\ast}\bJ_{\bbGamma}^{\phantom{\ast}}\phi_{j}, \phi_{k} \rangle_{\cH} 
        \\&= \frac{1}{\sqrt{\gamma_{j}\gamma_{k}}} \langle \bS_{\bbGamma} \phi_{j}, \phi_{k} \rangle_{\cH} 
       \\& = \frac{1}{\sqrt{\gamma_{j}\gamma_{k}}} \langle \bS_{\bbGamma}^{1/2} \phi_{j}, \bS_{\bbGamma}^{1/2} \phi_{k} \rangle_{\cH}
        \\&= \langle \phi_{j}, \phi_{k} \rangle_{\cH} 
    = \delta_{jk}.
    \end{split}
    \end{equation}

    Observe that when $k:\mathcal{X}\times\mathcal{X}\to\mathbb{R}$ is continuous, its RKHS $\cH$ is separable, by separability of $\cX$ itself. 
      Thus, 
    \begin{align*}
        \|\mathbf{H}\|_{\mathrm{HS}(\cH)}^2 &=
        \sum_{j,k\geq 1} \frac{1}{{\gamma_{j}\gamma_{k}}}\left\langle (\S_{\bbGamma} - \S_{\bbQ})\phi_{j}, \phi_{k}\right\rangle _{\cH}^2 
        \\&= \sum_{j,k\geq 1} \frac{1}{{\gamma_{j}\gamma_{k}}}\left\langle (\bI - \bM_{g}) \bJ_{\bbGamma}\phi_{j},\bJ_{\bbGamma}\phi_{k}\right\rangle _{L^2(\cX, \bbGamma)}^2 
        \\&= \sum_{j,k\geq 1} \left\langle (\bI - \bM_{g}) f_{j},f_{k}\right\rangle _{L^2(\cX, \bbGamma)}^2 
        \\&= \sum_{j,k\geq 1} \left\langle \bM_{g-1} f_{j},f_{k}\right\rangle _{L^2(\cX, \bbGamma)}^2 
        \\&= \|\bM_{g-1} \|_{\mathrm{HS}(L^2(\cX, \bbGamma))}^2.
    \end{align*}
    where $g = d\bbQ/d\bbGamma$ and we have made use of Lemma \ref{lem:Semb_is_multop} to pass to the second line. In summary, $\cN(0,\S_\bbQ)\sim\cN(0,\S_\bbP)$ implies that $\bM_{g-1}$ is Hilbert-Schmidt, and hence compact. Hence,  by Lemma \ref{lem:multisnotcompact} we have that $\bbP = \bbQ$.

    \medskip

    \noindent\textbf{General case.}
    Suppose now that in fact $\mathrm{supp}\{\bbGamma\}=\mathcal{X}_0\subset\mathcal{X}$. Note that $\mathcal{X}_0$ is also a locally compact Polish space as it is closed, and restricting $k$ to $\mathcal{X}_0$ preserves universality. Call the restricted kernel $k_0$, and let the corresponding RKHS be $\mathcal{H}_0$. 
    Let $\bbP_0,\bbQ_0,\bbGamma_0$ be the non-atomic Borel probability measures obtained by restricting $\bbP,\bbQ,\bbGamma$ to $\mathcal{X}_0$, respectively. Finally, let $\S_{\bbP_0},\S_{\bbQ_0}, \S_{\bbGamma_0}$ be the restricted covariance embeddings. Clearly, $\bbP=\bbQ\iff\bbP_0=\bbQ_0$. So, given the first part of our proof, it suffices to establish that $\mathcal{N}(0,\S_\bbGamma)\sim \mathcal{N}(0,\S_\bbQ)\implies \mathcal{N}(0,\S_{\bbGamma_0})\sim \mathcal{N}(0,\S_{\bbQ_0})$. To this aim, recall that from the RKHS restriction theorem (\cite[Corollary 5.8]{paulsen2016introduction} or \cite[Remark 2.2 in the Supplemental Material of ][]{waghmarepanaretos2022}), we know that $\mathcal{H}_0$ can be identified with the orthocomplement of $\mathcal{H}_1:=\{f\in\mathcal{H}: f|_{\mathcal{X}_0}=0\}\subseteq\mathcal{H}$.
     Therefore, that $f\in\mathcal{H}_1\Rightarrow f\in \mathrm{ker}(\S_\bbGamma)=\mathrm{ker}(\S_\bbGamma^{1/2})$, so that $\mathcal{H}_0\supset \overline{\mathcal{R}(\S_\bbGamma^{1/2})}=\mathrm{supp}\{\mathcal{N}(0,\S_\bbGamma)\}= \mathrm{supp}\{\mathcal{N}(0,\S_\bbQ)\}$. So  $\mathcal{N}(0,\S_\bbGamma)$ and $\mathcal{N}(0,\S_\bbQ)$ coincide with their restrictions on $\mathcal{H}_0$, which in turn are identified with $\mathcal{N}(0,\S_{\bbGamma_0})$ and $\mathcal{N}(0,\S_{\bbQ_0})$.\end{proof}

\begin{proof}[Proof of Corollary~\ref{cor:main}]
    The proof follows from the combination of Theorem~\ref{thm:main} and the Feldman-H\'{a}jek dichotomy. Indeed, one direction is trivial. For the other, observe that  the equivalence of the Gaussian measures $\cN(\m_{\bbP}, \S_{\bbP})$ and $\cN(\m_{\bbQ}, \S_{\bbQ})$ implies, by \eqref{eq:disintangled_FH}, the equivalence of the Gaussian measures $\cN(\mathrm{0}, \S_{\bbP})$ and $\cN(\mathrm{0}, \S_{\bbQ})$.
    However, by Theorem~\ref{thm:main}, this can only occur if $\bbP=\bbQ$.
\end{proof}

\begin{proof}[Proof of Proposition~\ref{prop:main:mean}]
We will first show that equivalence holds if and only if $\frac{d\bbP}{d\bbGamma}  - \frac{d\bbQ}{d\bbGamma}\in L^2(\cX,\bbGamma)$ for $\bbGamma = \frac{1}{2}(\bbP + \bbQ)$. To do so, we adapt the argument in  \cite[][]{eric2007testing} to (uncentered) second-order embeddings $\S_{\bbP},\S_{\bbQ}$. Noting that $\S_{\bbGamma}=\frac{1}{2}(\S_{\bbP} + \S_{\bbQ})$, we need to verify
    $
    \m_{\bbP} - \m_{\bbQ} \in \range\left((\S_{\bbP} + \S_{\bbQ})^{1/2}\right)
    $. 
    Equivalently,
    that following two conditions both hold true: 
    (1) 
    $
     \m_{\bbP} - \m_{\bbQ}\perp\ker(\S_{\bbP}+\S_{\bbQ}),$ and (2)
    $
     \sum_{j\geq 1} \gamma_j^{-1} \langle \phi_j, \m_{\bbP} - \m_{\bbQ} \rangle_{\cH}^2 < \infty,
    $
    where $(\gamma_k, \phi_k)_{k\geq 1}$ denotes the eigensystem of $\S_{\bbGamma}$. 

    The first condition is easily shown: taking $g\in \ker(\S_{\bbP}+\S_{\bbQ})$, and since $\bbGamma = \frac{1}{2}(\bbP + \bbQ)$ dominates both $\bbP$ and $\bbQ$,
    \begin{align*}
        \langle g, \m_{\bbP} - \m_{\bbQ} \rangle_{\cH} 
         &= \int_{\cX} g(x)\,d\bbP(x)-\int_{\cX} g(x)\,d\bbQ(x)
         \\&= \int_{\cX} g(x)\left( \frac{d \bbP}{d\bbGamma}(x) - \frac{d \bbQ}{d\bbGamma}(x)\right)d\bbGamma(x)
         \\&=\left\langle \frac{d \bbP}{d\bbGamma} - \frac{d \bbQ}{d\bbGamma}, \bJ_{\bbGamma} g \right\rangle _{L^2(\cX,\bbGamma)}.
    \end{align*}
         
    But if $g\in \ker(\S_{\bbP}+\S_{\bbQ})$, we have, 
    \begin{align*}
    \|\bJ_{\bbGamma}g\|^2_{L^2(\cX,\bbGamma)}
        &=\int_{\cX} g(x)^2 \, d\bbGamma(x) 
        \\&=  \frac{1}{2}\int_{\cX} g(x)^2 \, d\bbP(x) + \frac{1}{2}\int_{\cX} g(x)^2 \, d\bbQ(x) 
        \\&=  \langle g, (\S_{\bbP} + \S_{\bbQ}) g\rangle_{\cH} =0.
    \end{align*}
    Thus $\langle g, \m_{\bbP} - \m_{\bbQ} \rangle_{\cH} = 0$, as (1) requires. As for the second condition, a similar calculation as before yields
    \begin{align*}
        \sum_{j\geq 1} &\gamma_j^{-1} \langle \phi_j, \m_{\bbP} - \m_{\bbQ} \rangle_{\cH}^2 = \\
            &   =  \sum_{j\geq 1} \gamma_j^{-1}\left\langle \frac{d \bbP}{d\bbGamma} - \frac{d \bbQ}{d\bbGamma}, \bJ_{\bbGamma}\phi_j \right\rangle _{L^2(\cX,\bbGamma)}^2
            \\&=  \sum_{j\geq 1} \left\langle \frac{d \bbP}{d\bbGamma} - \frac{d \bbQ}{d\bbGamma}, \underset{f_j}{\underbrace{\gamma_j^{-1/2}\bJ_{\bbGamma}\phi_j }}\right\rangle _{L^2(\cX,\bbGamma)}^2.  
    \end{align*}
    Finally, observe that $f_j = \gamma_j^{-1/2}\bJ_{\bbGamma}\phi_j$ is an orthonormal basis of $L^2(\cX,\bbGamma)$, by the same reasoning as in \eqref{eq:fjisbasis}. 
    In summary, we have established that equivalence holds if and only if $\frac{d\bbP}{d\bbGamma}  - \frac{d\bbQ}{d\bbGamma}\in L^2(\cX,\bbGamma)$ for $\bbGamma = \frac{1}{2}(\bbP + \bbQ)$. But now observe that $\bbGamma$-almost everywhere on $\mathcal{X}$, $\left|\frac{d\bbP}{d\bbGamma}-\frac{d\bbQ}{d\bbGamma}\right|\leq \left|\frac{d\bbP}{d\bbGamma}\right|+\left|\frac{d\bbQ}{d\bbGamma}\right|=\frac{d\bbP}{d\bbGamma}+\frac{d\bbQ}{d\bbGamma}=\frac{d(\bbP+\bbQ)}{d\bbGamma}=\frac{d(2\bbGamma)}{d\bbGamma}=2,$
so squaring and integrating with respect to $\bbGamma$ yields $\left\|\frac{d\bbP}{d\bbGamma}-\frac{d\bbQ}{d\bbGamma}\right\|^2_{L^2(\mathcal{X},\bbGamma)}\leq 4$, regardless of whether $\bbP=\bbQ$ or   $\bbP\neq\bbQ$.
\end{proof}

For the proof of Theorem~\ref{thm:operat}, we recall some additional background related to Gaussian measures. Given  $\cN(\m_1,\S_1)\sim\cN(\m_2,\S_2)$, the \textit{Kullback-Leibler divergence} (or relative entropy) takes the explicit form:
\begin{equation}\label{eq:kl}
\begin{split}
 D_{\textnormal{KL}}&(\cN(\m_1,\S_1)\,||\,\cN(\m_2,\S_2)) = 
 \frac{1}{2}\|{\S_2^{-1/2}}(\m_1 - \m_2) \|^2 \\& \quad-  \frac{1}{2}\log\dettwo\left( \Id - \S_2^{-1/2}(\S_1 - \S_2)\S_2^{-1/2} \right).
\end{split}
\end{equation}
Here, $\dettwo$ denotes the Fredholm-Carleman determinant \cite{simon1977notes,fredholm1903classe} of a symmetric $\mathbf{H}$ with eigenvalues $\left\{\gamma_j\right\}_{j=1}^{\infty}$,
$$
\operatorname{det}_2(\mathbf{I}+\mathbf{H})=\prod_{j=1}^{\infty}\left(1+\gamma_j\right) e^{-\gamma_j}.
$$
It can be shown that the infinite product converges when $\sum_{j=1}^{\infty} \gamma_j^2<\infty$ and thus that the Carleman-Fredholm determinant is well-defined for all Hilbert-Schmidt operators with eigenvalues larger than $-1$. It is also known that the map $ \mathbf{H} \mapsto \operatorname{det}_2(\mathbf{I}+\mathbf{H})$ is strictly log-concave, continuous everywhere in $\|\cdot\|_{\textnormal{HS}(\cH)}$ norm and Gateaux differentiable on the subset of Hilbert-Schmidt operators whose spectrum excludes $-1$.

 \begin{proof}[Proof of Theorem~\ref{thm:operat}]
 Assume $\bbP\neq \bbQ$. Let $\{e_i\}_{i\geq 1}$ be a CONS comprised of eigenfunctions for $\S_\bbP$, with corresponding eigenvalue sequence $\{\lambda_i\}_{i\geq1}$. 
Define the sequence of projections
$
\cP_N = \sum_{i=1}^N e_i\otimes e_i
$
which converges strongly to the identity. 
Then, by \eqref{eq:kl}:
\begin{align*}
D_{\textnormal{KL}}\left({\cP_N}_{\#} \cN_\bbQ \,||\,{\cP_N}_{\#}\cN_{\bbP}\right) 
=& \: \frac{1}{2}\sum_{i= 1}^N \lambda^{-1}_i\langle  \m_\bbP - \m_\bbQ,e_i\rangle_{\cH}^2 \\ & \quad +  \frac{1}{2}\sum_{i=1}^N\left(\Delta_{i} - \log(1+\Delta_{i}) \right)
\end{align*}
where
$\Delta_i = \lambda_i^{-1}\langle(\S_{\bbQ} - \S_\bbP)e_i,e_i\rangle_{\cH}.
$ The first term is finite for any $N>0$, so we move our attention to the second. 
By \cite[][Proposition 3]{bach2022information}, we have that $\S_\bbP\prec C \S_\bbQ$ for some $C>0$. Hence:
\begin{align*}
1 + \Delta_{i} 
& = 1 + \frac{1}{\lambda_i}\langle(\S_{\bbQ} - \S_\bbP)e_i,e_i\rangle_{\cH} 
\\&= \frac{1}{\lambda_i}\langle\S_{\bbQ} e_i,e_i\rangle_{\cH} 
\\&    >  \frac{1}{C\lambda_i}\langle\S_{\bbP} e_i,e_i\rangle_{\cH} 
> 0
\end{align*}
This ensures the boundedness of each element in the sequence, and hence the finiteness of the projected relative entropy.
Then, similarly to the coercivity argument in \cite[][Proof of Lemma 3]{waghmare2023functionalgraphical}, we have that:
\begin{align*}
\sum_{i=1}^N
\left(\Delta_i - \log(1+\Delta_i) \right)
& = \log\left( \prod_{i=1}^N e^{\Delta_i}(1+ \Delta_i)^{-1}\right)
\\& = \sum_{i=1}^N\log\left(1+  \frac{1}{(1+\Delta_i)}\sum_{k=2}^{\infty} \frac{\Delta_i^k}{k!}\right)
\\& \geq  \log (1 + \frac{1}{3}\sum_{i=1}^N\Delta_{i}^2).
\end{align*}
However, note that necessarily the sum $\sum_{i=1}^N\Delta_{i}^2$ diverges as $N\to \infty$, since otherwise it would imply that the operator $\S_\bbP^{-1/2}(\S_\bbQ-\S_\bbP)\S_\bbP^{-1/2}$ is Hilbert-Schmidt; but we assumed that $\bbP\neq \bbQ$, and this directly contradicts Theorem~\ref{thm:main}.
 \end{proof}




\bibliographystyle{plain}
\bibliography{bib}

\end{document}

%% file: header.tex
\usepackage{comment}
\usepackage{fullpage}
\usepackage{subcaption}
\usepackage{graphicx}
\usepackage{enumitem}
\usepackage{xcolor}
\usepackage{hyperref}
\usepackage{multirow}
\usepackage{multicol}
\usepackage{caption} 
\usepackage{xfrac}  
\usepackage{graphicx}
\usepackage{tikz}
\usetikzlibrary{arrows.meta, positioning}
\usepackage{amsthm}
\usepackage{amsmath}
\usepackage{amsfonts}
\usepackage{amssymb}
\usepackage{euscript}
\usepackage{bm}
\usepackage{enumitem, hyperref}
\usepackage{authblk}
\usepackage[sort&compress,numbers]{natbib}

\makeatletter
\def\namedlabel#1#2{\begingroup
  #2%
  \def\@currentlabel{#2}%
  \phantomsection\label{#1}\endgroup
}
\makeatother

\newcommand{\trace}{\operatorname{trace}}
\newcommand{\range}{\mathcal{R}}
\newcommand{\Id}{\mathbf{I}}
\newcommand{\dettwo}{\operatorname{det}_2}
\renewcommand{\S}{\mathbf{S}}
\newcommand{\bzero}{\mathbf{0}}
\newcommand{\bA}{\mathbf{A}}
\newcommand{\bM}{\mathbf{M}}
\newcommand{\bU}{\mathbf{U}}
\newcommand{\bJ}{\mathbf{J}}
\newcommand{\bS}{\mathbf{S}}
\newcommand{\bH}{\mathbf{H}}
\newcommand{\bx}{\mathrm{x}}
\newcommand{\bu}{\mathrm{u}}
\newcommand{\bbP}{\mathbb{P}}
\newcommand{\bbQ}{\mathbb{Q}}
\newcommand{\bbN}{\mathbb{N}}
\newcommand{\bbR}{\mathbb{R}}
\newcommand{\cP}{\mathcal{P}}
\newcommand{\cH}{\mathcal{H}}
\newcommand{\cX}{\mathcal{X}}
\newcommand{\cN}{\mathcal{N}}
\newcommand{\bI}{\mathbf{I}}
\newcommand{\m}{\mathrm{m}}
\newcommand{\bbGamma}{{\mathpalette\makebbGamma\relax}}
\newcommand{\makebbGamma}[2]{%
  \raisebox{\depth}{\scalebox{1}[-1]{$\mathsurround=0pt#1\mathbb{L}$}}%
}

\DeclareMathOperator{\supp}{supp}

\theoremstyle{plain}
\newtheorem{definition}{Definition}[section]
\newtheorem{lemma}[definition]{Lemma}
\newtheorem{theorem}[definition]{Theorem}
\newtheorem{proposition}[definition]{Proposition}
\newtheorem{corollary}[definition]{Corollary}
\newtheorem{remark}[definition]{Remark}